%% file: arxiv.tex
\documentclass[twoside]{article}

\usepackage[accepted]{aistats2021}
\usepackage{makecell}
\usepackage{multirow}

\usepackage{mathtools}
\usepackage{amssymb}
\usepackage{amsthm}
\usepackage{amsmath}
\usepackage{soul}
\usepackage{color}

\DeclareMathOperator*{\argmin}{argmin}

\usepackage{algpseudocode}
\usepackage[ruled]{algorithm2e}
\SetKwRepeat{Do}{do}{while}

\algnewcommand{\comment}[1]{\Comment{{#1}}}

\theoremstyle{remark}

\include{defs}

\begin{document}

%

%

\twocolumn[

\aistatstitle{Revisiting Projection-free Online Learning: the Strongly Convex Case}

\aistatsauthor{ Dan Garber \And Ben Kretzu }

\aistatsaddress{ Technion - Israel Institute of Technology \And  Technion - Israel Institute of Technology } ]

\begin{abstract}
Projection-free optimization algorithms, which are mostly based on the classical Frank-Wolfe method, have gained significant interest in the machine learning community in recent years due to their ability to handle convex constraints that are popular in many applications, but for which computing projections is often computationally impractical in high-dimensional settings, and hence prohibit the use of most standard projection-based methods. In particular, a significant research effort was put on projection-free methods for online learning. In this paper we revisit the Online Frank-Wolfe (OFW) method suggested by \cite{Hazan12} and fill a gap that has been left unnoticed for several years: OFW achieves a faster rate of $O(T^{2/3})$ on strongly convex functions (as opposed to the standard $O(T^{3/4})$ for convex but not strongly convex functions), where $T$ is the sequence length. This is somewhat surprising since it is known that for offline optimization, in general, strong convexity does not lead to faster rates for Frank-Wolfe. We also revisit the bandit setting under strong convexity and prove a similar bound of  $\tilde O(T^{2/3})$ (instead of $O(T^{3/4})$ without strong convexity). Hence, in the current state-of-affairs, the best projection-free upper-bounds for the full-information and bandit settings with strongly convex and nonsmooth functions match up to logarithmic factors in $T$. 
\end{abstract}

\section{INTRODUCTION}

Computing projections onto convex sets is a fundamental computational primitive in most popular optimization methods such as projected gradient methods which are at the heart of numerous machine learning tasks. However, many machine learning applications involve optimization with structural constraints for which computing projections (e.g. Euclidean projection) is impractical in high-dimensional settings. It is for this reason that so-called \textit{projection-free} optimization methods, which replace the potentially computationally-expensive projection primitive with a different more efficient primitive, have attracted signifiant interest within the machine learning community in recent years. These projection-free methods are mostly based on the classical Frank-Wolfe method for constrained convex optimization (aka the conditional gradient method) \cite{Jaggi13, FrankWolfe, Polyak} which replaces the projection operation with a linear optimization step over the constraints. Indeed in many important cases such as constraints arising from combinatorial structure (e.g.,  paths/matchings/spanning trees in combinatorial graphs, or matroids), or from a low-rank matrix structure, linear optimization over the feasible set can be carried out very efficiently, sometimes by orders of magnitude faster than the projection operation, see for instance discussions in \cite{Jaggi13, Hazan12}.
Developing efficient projection-free methods, and in particular Frank-Wolfe-based methods for various  optimization paradigms central  to machine learning as thus become a popular research area where efforts have been focused on traditional offline optimization \cite{GH16, lacoste2015linear_fw, GH15} , stochastic optimization \cite{hazan2016variance, goldfarb2017linear, garber2019fast}, online learning \cite{Hazan12, Garber13, Karbasi19, garber2020improved}, and distributed computation \cite{bellet2015distributed, wang2016parallel}.

In online learning \cite{Cesa06, HazanBook}, which deals with sequential prediction over a (large) number of rounds, projection-free methods are of particular interest since naturally in such a setting, the response time of the online algorithms, i.e., the time it takes to compute a new prediction on each iteration, is of major importance in many applications. 
\cite{Hazan12} introduced the Online Frank-Wolfe method (OFW) for online convex optimization and proved that it attains $O(T^{3/4})$ regret, where $T$ is the number of prediction rounds, for convex loss functions in the full-information setting (i.e., after each round the loss function is fully observable to the learner), using a single linear optimization step over the feasible set per iteration. This is contrast to the Online Gradient Descent method (OGD) which attaines $O(\sqrt{T})$ regret, however requires $T$ projection steps. For the special case in which the feasible set is a polytope, \cite{Garber13} gave a modified Online Frank-Wolfe method with regret $O(\sqrt{nT})$, where $n$ is the dimension. Very recently \cite{HazanM20} presented a randomized online algorithm that also uses only a single linear optimization step per iteration, however one that is not based on the Frank-Wolfe method. Their algorithm attains $(\sqrt{n}T^{2/3})$ expected regret under the additional assumption that the loss functions are also smooth. 

In the bandit setting, in which only the scalar loss incurred by the online learner is observed after each round and not the loss function, \cite{Karbasi19} gave the first projection-free algorithm, which combines OFW and the gradient estimation idea of \cite{Flaxman05}, to obtain an algorithm with $O(nT^{4/5})$ expected regret. This was very recently improved by \cite{garber2020improved} by taking advantage of the special structure of the gradient estimator of \cite{Flaxman05}  and by considering a similar algorithm to that of \cite{Karbasi19} with the difference of considering the prediction round in blocks in order to obtain higher accuracy in the computation of the iterates, which led to a $O(\sqrt{n}T^{3/4}$) expected regret bound using overall $O(T)$ linear optimization steps in expectation, which matches (in terms of $T$) the current best upper-bound for the full-information setting (without additional assumptions such as smoothness) of \cite{Hazan12}. We also note that, besides the theoretical contributions, both \cite{Hazan12} and \cite{Karbasi19} have demonstrated the practical appeal of online Frank-Wolfe methods via extensive empirical studies.

While \cite{HazanM20} have managed to obtain a faster rate (at least in $T$) by leveraging the \textit{smoothness} of the loss functions, a property that is well known to allow for faster rates in offline and even stochastic optimization but is much less useful in online optimization, \textit{strong convexity}, another property of loss functions that is also well known to allow faster rates in convex optimization paradigms and in particular in online convex optimization \cite{HazanKKA06}, has remained unexplored in the context of projection-free methods. This is perhaps not surprising, since as discussed, most projection-free algorithms are based on the Frank-Wolfe method, and it is well known that, in general, the convergence rate of Frank-Wolfe for offline convex optimization does not improve in the presence of strong convexity \cite{Jaggi13, GH16}\footnote{Obtaining faster rates for Frank-Wolfe variants under strong convexity-like properties is an active research effort in recent years, however it mostly requires focousing on specific types of feasible sets, e.g., polytopes \cite{GH16} or strongly convex sets \cite{GH15}, etc., while here we are mainly interested in generic methods that apply to arbitrary convex and compact sets.}. Nevertheless, in this work we show, to the best of our knowledge, for the first time, that both in the full-information and bandit online settings, strong convexity does in fact lead to faster rates for the OFW algorithm \cite{Hazan12} and its bandit variant with blocks \cite{garber2020improved}. In particular we prove a $O(T^{2/3})$ regret bound for the full-information setting and a $O((nT)^{2/3}(\ln{T})^{1/3} + T^{2/3}\ln{T})$ expected regret bound for the bandit setting.

\begin{table*}[htb]\renewcommand{\arraystretch}{1.4}
{\small
\begin{center}
\caption{Comparison of regret bounds for state-of-the-art linear optimization-based online algorithms.
  }
  \begin{tabular}{| c | c | c | c | c |} \hline
    REFERENCE  & FEEDBACK &  SMOOTHNESS? & \makecell{STRONG \\ CONVEXITY?} & REGRET \\ \hline
\cite{Hazan12} & Full & x & x  & $T^{3/4}$\\ \hline
\cite{garber2020improved} & Bandit & x & x   & $\sqrt{n}T^{3/4}$\\ \hline
\cite{HazanM20} & Full & \checkmark & x   & $\sqrt{n}T^{2/3}$\\ \hline
This work (Thm. \ref{thm:main_fi_sc}) & Full & x  & \checkmark  & $T^{2/3}$\\ \hline
This work (Thm. \ref{thm:bandit_main_sc}) & Bandit & x  & \checkmark & $(nT)^{2/3}(\ln{T})^{1/3} + T^{2/3}\ln{T}$\\ \hline
  \end{tabular}
  \label{table:compare}
\end{center}
}
\vskip -0.2in
\end{table*}\renewcommand{\arraystretch}{1}

Quite pleasingly, our results do not hinge on any new particular technique, but mostly adapt those of \cite{Hazan12} to leverage the strong convexity of the losses. With this respect, beyond formally deriving the improved rates, we view our contribution as mainly of conceptual importance: observing that, as opposed to the offline setting, in the online setting strong convexity is indeed beneficial for Frank-Wolfe-based methods. 

On a slightly more technical level, an intuition to why strong convexity helps is as follows: in OFW, Frank-Wolfe is used to iteratively approximate (using one linear optimization step per iteration) the steps of the regret-optimal Regularized Follow the Leader method (RFTL), which on each iteration seeks to minimize the aggregated loss up to the current round plus an additional strongly convex regularization term. The reason for the slow rate of OFW is that the RFTL objective (i.e., subproblem that needs to be solved each iteration) drifts too much from one iteration to the next which does not allow for a good enough approximation. To control this drift one must increase the regularization beyond the optimal level which in turn leads to sub-optimal regret of $O(T^{3/4})$. Nevertheless, in the strongly convex case, due to the strong convexity, it can be shown that the drift from one round to the next in the RFTL objective is milder, which in turn allows for better approximation via a single Frank-Wolfe step. The same intuition, coupled with the recent result of \cite{garber2020improved}, also leads to the improved rate in the bandit setting.


\section{PRELIMINARIES}
\subsection{Online Convex Optimization}

In the online convex optimization with full information setting \cite{HazanBook}, an online learner is required to iteratively choose actions from a fixed feasible set $\mK\subseteq\reals^n$ which is assumed convex and compact.  After choosing his action $\x_t\in\mK$ on round $t\in[T]$ ($T$ is assumed to be known beforehand), the learner observes a loss function $f_t(\cdot)$, which is chosen by an adversary, and incurs the loss $f_t(\x_t)$, where $f_t:\mK\rightarrow\reals$ is a convex. 

The goal of the learner is to minimize the regret which is given by
\begin{align}
        \mathcal{R}_{T} := \sum_{t=1}^{T} f_t(\x_{t}) - \min\limits_{\x \in \mK} \sum_{t=1}^{T} f_t(\x). \label{eq:regret_def}
\end{align}

In the bandit-information setting, instead of observing the loss function after each iteration, the learner only observes his loss, that is the scalar value $f_t(\x_t)$. In this setting, we assume the adversary is oblivious, i.e., the loss functions $f_1,\dots,f_T$ are chosen beforehand and do not depend on the actions of the learner.
The bandit-feedback requires the learner to use random exploration and therefore, the goal is to minimize the expected regret $\mathbb{E}[\mathcal{R}_{T}]$.


We also make the following standard boundness assumptions, 
$\forall t\in[T]~\forall\x\in\mK~\forall\mathbf{g}\in\partial{}f_t(\x)$: $\Vert f_t \Vert_\infty = \sup_{\x \in \mK} |f_t(\x)| \leq M$ and $\Vert{\mathbf{g}}\Vert_2\leq G$, for some $M,G>0$. The latter implies each $f_t(\cdot)$ is $G$-Lipschitz over $\mK$.

Following \cite{Flaxman05}, we also assume the feasible set $\mK$ is full dimensional, contains the origin, and that there exist scalars $r,R>0$ such that $r\ball^n\subseteq\mK\subseteq{}R\ball^n$, where $\ball^n$ denotes the unit Euclidean ball centered at the origin in $\reals^n$.

\subsection{Additional  Notation And Definitions}

We denote by $\mS^n$ the unit sphere in $\mathbb{R}^n$, and we write $\mathbf{u} \sim S^n$ and $\mathbf{u} \sim \ball^n$ to denote a random vector $\u$ sampled uniformly from $\mS^n$ and $\ball^n$, respectively. We denote by $\Vert \mathbf{x} \Vert$ the $\ell_2$ norm of the vector $\mathbf{x}$. 

Finally, for a compact and convex set $\mK\subset\reals^n$, which satisfies the above assumptions (i.e., $r\ball^n\subseteq\mK\subseteq{}R\ball^n$), and a scalar $0 < \delta \leq r$, we define the set $\mK_{\delta} := (1-\delta/r)\mK = \{(1-\delta/r)\x~|~\x\in\mK\}$. In particular, it holds that $\mK_{\delta}\subseteq\mK$ and for all $\x\in\mK_{\delta}$, $\x+\delta\ball^n\subseteq\mK$.

We recall that a function $f: \mK \xrightarrow{} \mathbb{R}$  is $\alpha$-strongly convex over $\mathcal{K}$ if $\forall \mathbf{x}, \mathbf{y} \in \mathcal{K}$: $f(\mathbf{y}) \geq  f(\mathbf{x}) + \nabla f(\mathbf{x})^{\top} (\mathbf{y} - \mathbf{x}) + \frac{\alpha}{2}  \Vert \mathbf{y} - \mathbf{x} \Vert ^2$.
Let $\mathbf{x}^*$ be the unique minimizer of $f$, an $\alpha$-strongly convex function over $\mathcal{K}$. From the above definition and the first order optimality condition it follows that $\forall \mathbf{x} \in \mathcal{K}$:
\begin{align}
     \frac{\alpha}{2}  \Vert \mathbf{x} - \mathbf{x}^* \Vert ^2 \leq f(\mathbf{x}) - f(\mathbf{x}^*). \label{eq:strong_convexity}
\end{align}



\subsection{Smoothed Loss Functions}\label{sec:smooth}
As in \cite{Flaxman05, Karbasi19, garber2020improved}, our bandit algorithm is based on a randomized-smoothing of the loss functions technique.
We define the $\delta$-smoothing of a loss function $f$ by { $\hat{f}_{\delta} (\mathbf{x}) = \mathbb{E}_{\mathbf{u} \sim \ball^n} \left[ f(\mathbf{x} + \delta \mathbf{u}) \right]$ }.
We now cite several useful lemmas regarding smoothed functions.

\begin{lemma}[Lemma 2.6 in \cite{HazanBook}] \label{lemma:hazan_smooth}
    Let $f: \mathbb{R}^n \xrightarrow{} \mathbb{R}$ be $\alpha$-strongly convex and $G$-Lipschitz over a convex and compact set $\mK\subset\reals^n$. Then $\hat{f}_{\delta}$ is $\alpha$-strongly convex and $G$-Lipschitz over $\mK_{\delta}$, and $\forall \mathbf{x} \in \mathcal{K}_{\delta}$ it holds that $|\hat{f}_{\delta} (\mathbf{x}) - f(\mathbf{x})| \leq \delta G$. 
\end{lemma}
\begin{lemma}[Lemma 6.5 in \cite{HazanBook}] \label{lemma:hazan_gradient}
	$\hat{f}_{\delta}(\x)$ is differentiable and
    $\nabla \hat{f}_{\delta}(\mathbf{x}) = \mathbb{E}_{\mathbf{u} \sim \mS^n} \left[ \frac{n}{\delta} f(\mathbf{x} + \delta \mathbf{u})\mathbf{u} \right]$.
\end{lemma}

\begin{lemma} [see \cite{Bertsekas73}]
\label{lemma:bertsekas_grdient}
    Let $f: \mathbb{R}^n \xrightarrow{} \mathbb{R}$ be convex and suppose that all subgradients of $f$ are upper-bounded by $G$ in $\ell_2$-norm over a convex and compact set $\mK\subset\reals^n$. Then, for any $\x\in\mK_{\delta}$ it holds that $\Vert{\nabla{}\hat{f}_{\delta}(\x)}\Vert \leq G$.

\end{lemma}

\section{FULL-INFORMATION SETTING}
In this section we present and prove our main result - an improved $O(T^{2/3})$ regret bound for Online Frank-Wolfe (Algorithm \ref{alg:wo_blocks} below), in case all loss functions are $\alpha$-strongly convex for some $\alpha > 0$.

\begin{algorithm}
  \KwData{$T$, $\alpha$, $T_0$}
  $\x_1 \gets $ arbitrary point in $\mK$\\
  \For{$~ t = 1,\ldots,T ~$}{
    Play $\x_{t} $ and observe $f_t(\x_{t})$\\
    Set $\nabla_{t} \xleftarrow{} \nabla 
    f_t(\x_{t})$ \\
    Define $F_{t+1}(\x):=$ \makecell{$ \sum_{i=1}^{t}  \x^{\top} \nabla_{i} + \frac{\alpha}{2}\Vert\x - \x_{i}\Vert^2$ \\$ + T_0 \frac{\alpha}{2} \Vert\x - \x_1\Vert^2 $}\\
    $ \mathbf{v}_t \in \argmin\limits_{\x \in \mK} \{ \nabla F_{t+1}(\x_t)^{\top} \cdot \x \} $\\
	$ \sigma_{t} = \argmin\limits_{\sigma \in [0, 1]}  \{ F_{t+1}( \x_t + \sigma (\mathbf{v}_t - \x_t ) ) \}$\comment{\makecell{ Line-\\search}}\\
	$ \x_{t+1} = \x_t + \sigma_{\tau} (\mathbf{v}_t - \x_t) $ 
 }
  \caption{Online Frank-Wolfe (see also \cite{Hazan12})}\label{alg:wo_blocks}
\end{algorithm}

\begin{theorem}  
\label{thm:main_fi_sc}
    Suppose all loss functions are $\alpha$-strongly convex for some $\alpha > 0$. Setting $T_0 = \max \Big{\{} 1, \frac{2 (G+2R\alpha) \sqrt{b}}{\alpha R^2} \Big{\}}$ with $b = \max \Big{\{}  \left( \frac{G+2R\alpha}{\alpha} \right)^2,  \frac{8 (2R)^2 (G+2R\alpha)}{\alpha} \Big{\}}$ in Algorithm \ref{alg:wo_blocks},  guarantees that the regret is upper-bounded by
{
    \begin{align*}
        \mathcal{R}_{T} & \leq   4 \frac{ \left(  G + 2  R \alpha \right)^2}{\alpha} \ln{(T)} +  2 \alpha R^2 \\
        & + 10  \frac{(G+2R\alpha)^2}{\alpha}  + 16 (G+2R\alpha)^\frac{4}{3}  \frac{  R^\frac{2}{3}}{\alpha^\frac{1}{3}} \nonumber \\
        & +  4  G \left( \frac{G+2R\alpha}{\alpha} \right) T^\frac{2}{3} + 8 \sqrt{2}  \frac{  (G+2R\alpha)^\frac{1}{3} R^\frac{2}{3}}{\alpha^\frac{1}{3}} T^\frac{2}{3},
    \end{align*}}
    and the overall number of calls to the linear optimization oracle is $T$ (one per iteration).
\end{theorem}

\subsection{Proof of Theorem \ref{thm:main_fi_sc}} 
For the regret analysis we need a well known lemma known as the "Follow-the-leader-Be-the-leader"  lemma, which we state here in a slightly modified version. A proof is given in the appendix for completeness.
\begin{lemma}\label{lemma:FTL_BTL}
Let $\mK \subseteq \reals^d$ a convex and compact set, { $\{ g_m (\x)  \}_{m=1}^{T}$} a set of convex functions, $\x_1 \in \reals^d$ and $c_1 \in \reals^+$. Denote {\small $\x_{\tau}^* = \argmin\limits_{\x \in \mK} \Big{\{}  \sum_{m=1}^{\tau-1} g_m (\x) +  c_1 \Vert \x - \x_1 \Vert^2 \Big{\}} $}  for every {$\tau \in [1, T+1]$}. Then for every $\x \in \mK$ we have that 
{ \begin{align*}
    \sum_{m=1}^{T} \big( g_m( \x_{m}^*) - g_m( \x) \big) \leq & \sum_{m=1}^{T} \big(g_m( \x_{m}^*) - g_m( \x_{m+1}^*) \big) \\
    & + c_1 \Vert \x - \x_1 \Vert^2.
\end{align*}}
\end{lemma}

For the purpose of the analysis of Algorithm \ref{alg:wo_blocks}, we define the auxiliary sequence $\{\x^*_t\}_{t=1}^{T}$ as $\x_t^* = \arg\min_{\x\in\mK}F_t(\x)$, where $F_t(\cdot)$ is as defined in Algorithm \ref{alg:wo_blocks}. Note $\x_t^*$ is simply the point played by the well-known (exact) Regularized Follow the Leader (RFTL) method \cite{HazanBook}. 

The following lemma upper-bounds the regret of Algorithm  \ref{alg:wo_blocks} in terms of how well does the iterate $\x_t$, which is obtained by applying a single Frank-Wolfe step to the RFTL objective $F_t(\cdot)$, approximates the optimal value of the RFTL objective, whose minimizer is the RFTL iterate $\x_t^*$. We recall this sequence of approximation errors is captured by the sequence $\{\epsilon_t\}_{t\geq 1}$. 

\begin{lemma} \label{lemma:full_regret}
Let $\{\epsilon_{t}\}_{t=1}^{T} > 0$ and $T_0 \geq 1$. Suppose that throughout the run of Algorithm \ref{alg:wo_blocks}, for all time steps $t = 1, \dots, T$ it holds that $F_{t}(\x_{t}) -  F_{t}(\x_{t}^*) \leq \epsilon_{t}$. Then,
the regret of Algorithm \ref{alg:wo_blocks} is upper-bounded by
{
\begin{align*}
    \mathcal{R}_{T} \leq &  \frac{2 \left(  G + 2  R \alpha \right)^2}{\alpha} (1 + \ln{(T)}) +  2 \alpha R^2 T_0  \\
    & +  G \sum_{t=1}^{T} \sqrt{\frac{2 \epsilon_{t} }{\alpha(t-1+T_0)}}.
\end{align*}}
\end{lemma}

\begin{proof}
Using the definition of the regret, Eq.\eqref{eq:regret_def}, and the $\alpha$-strongly convexity of each $f_{t}(\x)$, we have that
{
\begin{align}
    & \mathcal{R}_{T}  \leq \sum_{t=1}^{T} \nabla_t^\top \left( \x_{t} - \x^* \right) - \frac{\alpha}{2} \Vert \x_{t} - \x^* \Vert^2 \nonumber \\ 
    & ~~ = \sum_{t=1}^{T} \nabla_{t}^{\top} \left( \x_{t} - \x_{t}^* +  \x_{t}^* - \x^* \right) - \frac{\alpha}{2} \Vert \x_{t} - \x^* \Vert^2.  \label{eq:full_regret_rftl_lemma_sc}
\end{align}}

Using Cauchy-Schwarz inequality, $\Vert  \nabla_{t} \Vert \leq G$, the fact that $F_{t}(\x)$ is $(t-1+T_0)\alpha$-strongly convex, Eq.\eqref{eq:strong_convexity}, and the assumption that $F_{t}(\x_{t}) - F_{t}(\x_{t}^*) \leq \epsilon_t$, we have 
{
\begin{align}
    \sum_{t=1}^{T}  \nabla_{t}^{\top} (\x_{t} - \x_{t}^*) & \leq  G \sum_{t=1}^{T} \Vert  \x_{t} - \x_{t}^* \Vert \nonumber \\
    & \leq G \sum_{t=1}^{T} \sqrt{\frac{2 \left( F_{t}(\x_{t}) - F_{t}(\x_{t}^*) \right) }{\alpha(t-1+T_0)}}   \nonumber  \\
    & \leq  G \sum_{t=1}^{T} \sqrt{\frac{2 \epsilon_{t} }{\alpha(t-1+T_0)}}. \label{eq:regret_rftl_lemma_first_term_sc}
\end{align}}

Now, we need to obtain an upper bound on $\nabla_{t}^{\top} \left( \x_{t}^* - \x^* \right) - \frac{\alpha}{2} \Vert \x_{t} - \x^* \Vert^2 $. We will start with a few preliminary steps. \newline
For all $\x, \y, \mathbf{z} \in \mK$, the following holds
{ \begin{align}
    \Vert \x - \mathbf{z} \Vert^2 - \Vert \y - \mathbf{z} \Vert^2 & \leq  \Vert\x + \y - 2  \mathbf{z} \Vert ~ \Vert \x - \y\Vert \nonumber \\
     & \leq 4R \Vert \x - \y \Vert. \label{eq:gap_of_square_norm}
\end{align}}

Define $g_t(\x) =  \x^\top \nabla_{t} + \frac{\alpha}{2} \Vert \x - \x_{t} \Vert^2$, using Cauchy-Schwarz inequality and Eq.\eqref{eq:gap_of_square_norm}, we have for all $ \x,\y \in \mK$,
{ \begin{align}
     g_t (\x) - g_t(\y) & \leq  \left( G + 2 R \alpha \right)  \Vert \x - \y \Vert . \label{eq:Lip_g_tilde_sc}
\end{align} }

Since for every $t$, $F_{t+1}(\x)$ is $(t+T_0)\alpha$-strongly convex,  using Eq. \eqref{eq:strong_convexity},  we have that
{ \begin{align*}
    \frac{(t + T_0)\alpha}{2}  & \Vert \x_{t}^* -  \x_{t+1}^*\Vert ^2 
    \leq   F_{t+1}  ( \x_{t}^*  ) - F_{t+1}(\x_{t+1}^*) \nonumber  \\
    & ~~~~ =  F_{t} ( \x_{t}^*  ) -  F_{t}(\x_{t+1}^*) + g_t( \x_{t}^* ) - g_t( \x_{t+1}^* ) \nonumber \\
    & ~~~~ \leq  g_t( \x_{t}^* )  -  g_t( \x_{t+1}^* ) . 
\end{align*}}

The last inequality is since fot every $t$, $F_{t}(\x_{t}^*) \leq F_{t}(\x_{t+1}^*)$.
From the two last equations, we obtain
{
\begin{align}
    \Vert \x_{t}^* - \x_{t+1}^* \Vert & \leq \frac{2 \left( G + 2 R \alpha \right) }{(t + T_0)\alpha} . \label{eq:absolut_dist_sc}
\end{align}}

From Eq. \eqref{eq:Lip_g_tilde_sc} and Eq. \eqref{eq:absolut_dist_sc}, we have that
{
\begin{align}
    \sum_{t=1}^{T} g_t(\x_{t}^*)  - g_t  (\x_{t+1}^*) & \leq \sum_{t=1}^{T} \frac{2  \left( G + 2 R \alpha \right)^2 }{(t + T_0)\alpha} \nonumber \\ 
    \leq & \frac{2  \left(  G + 2  R \alpha \right)^2}{\alpha}  (1 + \ln{(T)}) \label{eq:sum_g_minimizers_sc} .
\end{align}}

The last inequality is since $ \sum_{t=1}^{T} \frac{1 }{t + T_0}  \leq 1 + \ln{(T)}$. 
Using Lemma \ref{lemma:FTL_BTL}, for all $\x \in \mK$, we have that
{ \begin{align}
    \sum_{t=1}^{T} g_t(\x_{t}^*) - \sum_{t=1}^{T} g_t(\x) \leq & \sum_{t=1}^{T} g_t(\x_{t}^*) - g_t(\x_{t+1}^*) \nonumber \\
    & ~~ + \frac{T_0 \alpha}{ 2 } \Vert \x - \x_1 \Vert^2 \label{eq:ftl_btl}
\end{align}}

By definition of $g_t(\x)$, we have
{
\begin{align*}
    & \sum_{t=1}^{T}  \nabla_{t}^{\top} (\x_{t}^* - \x^*) - \frac{\alpha}{2} \Vert \x_{t} - \x^* \Vert^2 \leq  \\
    &  \leq \sum_{t=1}^{T} \nabla_{t}^{\top} \x_{t}^* + \frac{\alpha}{2} \Vert  \x_{t}^* - \x_{t} \Vert^2 - \nabla_{t}^{\top} \x^* - \frac{\alpha}{2}   \Vert  \x^* - \x_{t} \Vert^2  \\
    &  = \sum_{t=1}^{T} g_t(\x_{t}^*) - g_t(\x^*). 
\end{align*} }
Thus, applying Eq.\eqref{eq:ftl_btl} and Eq.\eqref{eq:sum_g_minimizers_sc}, we obtain that
{
\begin{align*}
    \sum_{t=1}^{T}  \nabla_{t}^{\top}  (\x_{t}^*  - & \x^*) - \frac{\alpha}{2} \Vert \x_{t} - \x^* \Vert^2 \leq \nonumber \\
    & \leq  \sum_{t=1}^{T} g_t(\x_{t}^*) - g_t(\x_{t+1}^*) + \frac{T_0 \alpha}{ 2 } \Vert \x^* - \x_1 \Vert^2 \nonumber \\
    & \leq \frac{2 \left(  G + 2  R \alpha \right)^2}{\alpha} (1 + \ln{(T)}) +  2 \alpha R^2 T_0. 
\end{align*}}

Plugging-in the above bound together with Eq.\eqref{eq:regret_rftl_lemma_first_term_sc} in Eq.\eqref{eq:full_regret_rftl_lemma_sc}, we obtain the result of the lemma:
{
\begin{align*}
    \mathcal{R}_{T}  \leq & \frac{2 \left(  G + 2  R \alpha \right)^2}{\alpha} (1 + \ln{(T)}) +  2 \alpha R^2 T_0 \nonumber \\
    & +  G \sum_{t=1}^{T} \sqrt{\frac{2 \epsilon_{t} }{\alpha(t-1+T_0)}}.
\end{align*}}
\end{proof}

In order to complement Lemma \ref{lemma:full_regret}, the following lemma sets an upper-bound on the sequence of approximation errors $\{\epsilon_t\}_{t\geq 1}$ which can be guaranteed throughout the run of Algorithm \ref{alg:wo_blocks} using a single Frank-Wolfe step on each iteration (i.e., single call to the linear optimization oracle per iteration).

\begin{lemma} \label{lemma:epsilon_gap}
Let $ \Big{\{} \epsilon_{t} = b\alpha(t+T_0)^{1/3} \Big{\}}_{t=1}^{T} $, when $T_0 = \max \Big{\{} 1, \frac{2 (G+2R\alpha) \sqrt{b}}{\alpha R^2} \Big{\}}$ and $b = \max \Big{\{}  2\left( \frac{G+2R\alpha}{\alpha} \right)^2,  \left(\frac{64  (G+2R\alpha)R^2}{\alpha}\right)^\frac{2}{3} \Big{\}}$. Then, throughout the run of Algorithm \ref{alg:wo_blocks}, for all time steps $t = 1, \dots, T$ it holds that $F_{t}(\x_{t}) -  F_{t}(\x_{t}^*) \leq \epsilon_{t}$.
\end{lemma}

\begin{proof}
We will prove this lemma by induction. We will first start with a few preliminary steps.

Since the step-size $\sigma_{t} \in [0,1]$ is chosen via line-search in Algorithm \ref{alg:wo_blocks} , we have that 
\begin{align*}
    F_{t+1}(\x_{t+1})  = F_{t+1}(\x_{t} + \sigma_{t}(\mathbf{v}_{t} - \x_{t})) \leq  F_{t+1}(\x_t). 
\end{align*}

Also, from the convexity of $F_t(\cdot)$ and the choice of $\v_t$, it follows that
{ \begin{align*}
     F_{t+1}(\x_t) - F_{t+1}(\x_{t+1}^*)  & \leq \nabla F_{t+1}(\x_t)^{\top} (\mathbf{x}_{t} - \x_{t+1}^*) \nonumber \\
     & \leq \nabla F_{t+1}(\x_t)^{\top} (\mathbf{x}_{t} - \mathbf{v}_{t}).
\end{align*}}

Then, from the last two observations, we obtain
{
\begin{align}
     F_{t+1}(\x_{t+1}) - F_{t+1}(\x_{t+1}^*) \leq
     \nabla F_{t+1}(\x_t)^{\top} (\mathbf{x}_{t} - \mathbf{v}_{t}) \label{eq:need_name}
\end{align}}

Since for every $t$, $F_t(\x)$ is $\alpha(t-1+T_0)$-strongly convex, using Eq.\eqref{eq:strong_convexity}, we have that for any $\x \in \mK$,
{ \begin{align}
    \Vert \x - \x_t^* \Vert^2 \leq \frac{2\left( F_t(\x) - F_t(\x_t^*) \right)}{\alpha (t-1+T_0)} \label{eq:F_t_strong_convexity}
\end{align}}

Now we start the proof using induction.\\
\emph{Induction base}: for $\tau =1 $, we have
\begin{align*}
    F_{1}(\x_{1}) -  F_{1}(\x_{1}^*) = 0 - T_0\frac{\alpha}{2} \Vert \x_{1}^* - \x_1 \Vert^2 \leq 0 \leq \epsilon_1.    
\end{align*}

\emph{Induction assumption}: for $ \tau = t$ it holds that 
\begin{align*}
    F_{t}(\x_{t}) -  F_{t}(\x_{t}^*) \leq \epsilon_t.
\end{align*}

\emph{Induction step}: Let $\tau = t+1$ and we need to show that $F_{t+1}(\x_{t+1}) -  F_{t+1}(\x_{t+1}^*) \leq \epsilon_{t+1}$. We start by bounding the initial gap, $ F_{t+1}(\x_{t}) - F_{t+1}(\x_{t+1}^*)$ and then by bounding the improvement step $F_{t+1}(\x_{t+1}) -  F_{t+1}(\x_t)$. Finally, we will combine them.\newline
Since $F_{t}(\x_{t+1}^*) \geq F_{t}(\x_{t}^*)$, using Eq.\eqref{eq:Lip_g_tilde_sc}, we have
\begin{align*}
    F_{t+1}(\x_{t}) & - F_{t+1}(\x_{t+1}^*)  = \nonumber \\
    & = F_{t}(\x_{t}) - F_{t}(\x_{t+1}^*) + g_t (\x_{t}) - g_t(\x_{t+1}^*) \nonumber \\
    & \leq F_t(\x_t) - F_{t}(\x_{t}^*) + \left( G + 2 R \alpha \right)  \Vert \x_t - \x_{t+1}^* \Vert. 
\end{align*}
Using the induction assumption and the Cauchy-Schwarz inequality, we have
{
\begin{align}
    & F_{t+1}(\x_{t}) - F_{t+1}(\x_{t+1}^*)  \leq \nonumber \\
    & \leq \epsilon_t + \left( G + 2 R \alpha \right)   \Vert \x_t - \x_{t}^* \Vert + \left( G + 2 R \alpha \right)  \Vert \x_t^* - \x_{t+1}^* \Vert   \nonumber \\
    & \leq \epsilon_t +   \frac{\left( G + 2 R \alpha \right) \sqrt{2  \epsilon_{t}} }{\sqrt{\alpha(t-1+T_0)}} +   \frac{2 \left( G + 2 R \alpha \right)^2 }{(t + T_0)\alpha} \label{eq:initial_gap_1}.
\end{align}}

The last inequality follows using Eq.\eqref{eq:absolut_dist_sc}, Eq.\eqref{eq:F_t_strong_convexity}, and the induction assumption.

We now show that on the RHS of Eq.\eqref{eq:initial_gap_1}, the second term is larger than the third term. Since $\epsilon_t = b\alpha(t+T_0)^{1/3}$, for every $t$ it holds that
{
\begin{align*}
    \sqrt{\frac{2 \epsilon_{t} }{\alpha(t-1+T_0)}}  = \sqrt{\frac{2b(t+T_0)^{1/3} }{(t-1+T_0)}} & \geq  \frac{\sqrt{2b} }{(t-1+T_0)^{1/3}}  \\
    & \geq    \frac{2 \left( G + 2 R \alpha \right) }{(t + T_0)\alpha}
\end{align*}}
The last inequality is since $b \geq 2\left( \frac{G+2R\alpha}{\alpha} \right)^2$,  $T_0 \geq 1$.

Then, plugging-in the value of $\epsilon_t$ in Eq.\eqref{eq:initial_gap_1}, using the last result, and since {\small $\frac{(t+T_0)^{1/3} }{(t-1+T_0)} \leq  \frac{ 2^\frac{1}{3} }{(t-1+T_0)^{2/3}}$ }, we obtain an upper bound for the initial gap:
{
\begin{align}
    & F_{t+1}(\x_{t}) - F_{t+1}(\x_{t+1}^*) \leq \epsilon_t +   \frac{4 \left( G + 2 R \alpha \right)\sqrt{ b} }{(t-1+T_0)^{\frac{1}{3}}}. \label{eq:initial_gap}
\end{align}}

Now we will analyze the improvement step. For our analysis we define the optimal step-size $\hat{\sigma}_t = \min \Big{\{} \frac{  \nabla F_{t+1}(\x_t)^{\top} (\x_{t} - \mathbf{v}_{t})}{\alpha (t+T_0) (2R)^2}, 1 \Big{\}}$. 
Since $\sigma_{t}$ is chosen via line-search, we have that 
\begin{align*}
    F_{t+1}(\x_{t+1}) \leq  F_{t+1}(\x_t + \hat{\sigma}_{ t}(\mathbf{v}_{t} - \x_t)).
\end{align*}

Since $F_{t+1}(\mathbf{x})$ is $\alpha(t+T_0)$-smooth, it holds that 
{ \begin{align*}
    F_{t+1}(\x_{t+1}) -  F_{t+1}  (\x_t) \leq & \hat{\sigma}_{t} \nabla F_{t+1}(\x_t)^{\top} (\mathbf{v}_{t} - \x_t) \nonumber \\
    & + \frac{\alpha(t+T_0)}{2} \hat{\sigma}_{t}^2  \Vert \mathbf{v}_{t} - \x_t \Vert ^2. 
\end{align*} }

We now consider several cases.
\\ \emph{Case 1}: If $\nabla F_{t+1}(\x_t)^{\top} (\mathbf{x}_{t} - \mathbf{v}_{t}) \leq \epsilon_{t+1}$, then from Eq. \eqref{eq:need_name} we have 
\begin{align*}
    F_{t+1}(\x_{t+1}) -  F_{t+1}(\x_{t+1}^*) \leq \epsilon_{t+1}. 
\end{align*}
\emph{Case 2}: Else, $\nabla F_{t+1}(\x_t)^{\top} (\mathbf{x}_{t} - \mathbf{v}_{t}) \geq \epsilon_{t+1}$, then we have two sub cases: \\
\emph{Case 2.1}: If $ \alpha (t+T_0) (2R)^2 \leq \nabla F_{t+1}(\x_t)^{\top} (\x_{t} - \mathbf{v}_{t})$, then  $\hat{\sigma}_t = 1$ and we have
{ \begin{align}
    & F_{t+1}(\x_{t+1}) -  F_{t+1}(\x_t)  \leq  - \frac{\alpha(t+T_0)(2R)^2}{2}. \label{eq:improvment_case_2_1}
\end{align}}

\emph{Case 2.2}: Else, $ \alpha (t+T_0) (2R)^2 \geq \nabla F_{t+1}(\x_t)^{\top} (\x_{t} - \mathbf{v}_{t})$, then $\hat{\sigma}_t = \frac{  \nabla F_{t+1}(\x_t)^{\top} (\x_{t} - \mathbf{v}_{t})}{\alpha (t+T_0) (2R)^2}$ and we have
{
\begin{align}
    F_{t+1}(\x_{t+1}) -  F_{t+1}(\x_t)  & \leq  - \frac{ \left( \nabla F_{t+1}(\x_t)^{\top} (\x_{t} - \mathbf{v}_{t}) \right)^2 }{2 \alpha (t+T_0) (2R)^2} \nonumber \\
    & \leq - \frac{ \epsilon_{t+1}^2 }{2 \alpha (t+T_0) (2R)^2} \label{eq:improvment_case_2_2}
\end{align}} \newline
Combining Eq.\eqref{eq:initial_gap}, Eq.\eqref{eq:improvment_case_2_1}, Eq.\eqref{eq:improvment_case_2_2} and plugging-in the value of $\epsilon_{t+1}$, we have
{
\begin{align}
    &F_{t+1}(\x_{t+1})  -
     F_{t+1}(\x_{t+1}^*)  \leq  \epsilon_t + \frac{4 \left( G + 2 R \alpha \right)\sqrt{ b} }{(t-1+T_0)^{\frac{1}{3}}}  \nonumber \\ 
    & ~~~ - \min \bigg{\{} 2 \alpha(t+T_0) R^2, \frac{ b^2\alpha(t+1+T_0)^{2/3} }{8  (t+T_0) R^2} \bigg{\}} \label{eq:final_gap}
\end{align}}

Since $T_0 \geq \frac{2(G+2R\alpha)}{\alpha R^2} \sqrt{b}$, we have $2 \alpha(t+T_0) R^2 \geq \frac{4 \left( G + 2 R \alpha \right)\sqrt{ b} }{(t-1+T_0)^{\frac{1}{3}}}$ and also, since $b \geq \left(\frac{64(G+2R\alpha)R^2}{\alpha}\right)^\frac{2}{3}$, we have $\frac{ b^2\alpha(t+1+T_0)^{2/3} }{8  (t+T_0) R^2} \geq \frac{4 \left( G + 2 R \alpha \right)\sqrt{ b} }{(t-1+T_0)^{\frac{1}{3}}}$. From both, we have that 
{\small\begin{align*}
    \frac{4 \left( G + 2 R \alpha \right)\sqrt{ b} }{(t-1+T_0)^{\frac{1}{3}}}  \leq  \min \bigg{\{} 2 \alpha(t+T_0) R^2, \frac{ b^2\alpha(t+1+T_0)^{2/3} }{8  (t+T_0) R^2} \bigg{\}}. 
\end{align*}}
\newline Then, from Eq.\eqref{eq:final_gap} and since $\epsilon_t \leq \epsilon_{t+1}$, we finally obtain $    F_{t+1}(\x_{t+1})  - F_{t+1}(\x_{t+1}^*)  \leq  \epsilon_t \leq \epsilon_{t+1}$.
\end{proof}
With all technical ingredients set in place, we can now prove Theorem \ref{thm:main_fi_sc}.
\begin{proof}[Proof of Theorem \ref{thm:main_fi_sc}] Using Lemma \ref{lemma:full_regret}, Lemma \ref{lemma:epsilon_gap}, and plugging-in $T_0 = \max \Big{\{} 1, \frac{2 (G+2R\alpha) \sqrt{b}}{\alpha R^2} \Big{\}}$ and $\big{\{}\epsilon_t = b\alpha(t+T_0)^{1/3} \big{\}}_{t=1}^{T}$, we have
{
\begin{align*}
    \mathcal{R}_{T} \leq &  \frac{2 \left(  G + 2  R \alpha \right)^2}{\alpha} (1 +\ln{(T)})  + 4 (G+2R\alpha) \sqrt{b} \\
    &  +  2 \alpha R^2  + \sqrt{2}  G \sqrt{b} \sum_{t=1}^{T} \sqrt{\frac{(t+T_0)^{1/3} }{(t-1+T_0)}}\\
    \leq &  \frac{2 \left(  G + 2  R \alpha \right)^2}{\alpha} (1 +\ln{(T)})  + 4 (G+2R\alpha) \sqrt{b} \nonumber \\
    & +  2 \alpha R^2 +  2\sqrt{2}  G \sqrt{b} T^\frac{2}{3}
\end{align*}}

Plugging in $b = \max \Big{\{}  2\left( \frac{G+2R\alpha}{\alpha} \right)^2,  \left(\frac{64  (G+2R\alpha)R^2}{\alpha}\right)^\frac{2}{3} \Big{\}}$, we obtain the regret bound in the theorem.
\end{proof}

\section{BANDIT SETTING}
In this section we present our improved bound for the bandit-information setting. Our algorithm is simply that of \cite{garber2020improved} and is presented below as Algorithm \ref{alg:Bandit}. Due to lack of space, we refer the interested reader to \cite{garber2020improved} for more details.

\begin{algorithm}[ht]
  \KwData{$T$, $r, R$, $K$, $\alpha$, $\delta\in(0,r]$, $\{\epsilon_{m}\}_{m=1}^{T/K}$, $T_0$}
  $\mathbf{x}_0 \gets $ arbitrary point in $\mK_{\delta}$, $\x_1 \gets \x_0$\\
  \For{$~ m = 1,\ldots,\frac{T}{K} ~$}{
    define  $\hat{F}_m(\mathbf{x}):= \makecell{
        \sum_{i=1}^{m-1}  \mathbf{x}^{\top} \hat{\mathbf{g}}_i + K \frac{\alpha}{2}\Vert\mathbf{x} - \mathbf{x}_{i-1}\Vert^2  \\+ T_0 \frac{\alpha}{2} \Vert\mathbf{x} - \mathbf{x}_1\Vert^2}$
    \\
    \If{$m>1$}{ 
    run Algorithm \ref{alg:Cg} with set $\mathcal{K}_{\delta}$,  tolerance $\epsilon_{m}$, initial vector $\mathbf{x}_{m-1}$, and function $\hat{F}_{m}(\mathbf{x})$. Execute \textbf{in parallel} to the following \textbf{for} loop over $s$
    }
    \For{$~ s = 1 ,\ldots, K ~$}{
    $\mathbf{u}_t$ $\sim S^n$ \comment{$t=(m-1)K+s$}\\
    play $\mathbf{y}_t \xleftarrow{} \mathbf{x}_{m-1} + \delta \mathbf{u}_t$ and observe $f_t(\mathbf{y}_t)$\\
    $\mathbf{g}_t$ $\xleftarrow{} \frac{n}{\delta} f_t(\mathbf{y}_t) \mathbf{u}_t$
    }
    $\hat{\mathbf{g}}_m$ $\xleftarrow{}$ $\sum_{s=1}^{K} \mathbf{g}_{(m-1)K+s}$\\
    \If{$m>1$}{
   $\mathbf{x}_{m} \gets$ output of Algorithm \ref{alg:Cg}
    }
  }
  \caption{Online Bandit Frank-Wolfe (see \cite{garber2020improved})}\label{alg:Bandit}
\end{algorithm}
\begin{algorithm}[!]
  \KwData{$\epsilon$, $\x_{in}$, $F_m(\x)$}
  $\mathbf{z}_1 \gets \x_{in}$, $\tau \gets 0 $\\
  \Do{$\nabla F_m(\mathbf{z}_\tau)^{\top} (\mathbf{z}_{\tau} - \mathbf{v}_{\tau}) > \epsilon$}{
    $\tau \gets \tau + 1 $\\
    $ \mathbf{v}_\tau \in \argmin\limits_{\x \in \mK} \{ \nabla F_m(\mathbf{z}_\tau)^{\top} \cdot \x \} $\\
	{\small$ \sigma_{\tau} = \argmin\limits_{\sigma \in [0, 1]}  \{ F_m(\mathbf{z}_\tau + \sigma (\mathbf{v}_\tau - \mathbf{z}_\tau)) \}$\comment{Line-search}}\\
	$ \mathbf{z}_{\tau+1} = \mathbf{z}_\tau + \sigma_{\tau} (\mathbf{v}_\tau - \mathbf{z}_\tau) $ \comment{$\mathbf{z}_{\tau+1} \in \mK $}
  }
  $\x_{out} \gets \mathbf{z}_{\tau}$
  \caption{Frank-Wolfe with Stopping Condition}\label{alg:Cg}
\end{algorithm}

\begin{theorem} \label{thm:bandit_main_sc}
Suppose all loss functions are $\alpha$-strongly convex for some $\alpha>0$. For all $ c > 0 $ such that $\frac{cT^{-1/3}}{r} \leq 1$, setting $\delta = c T^{-\frac{1}{3}} $, $K = T^{\frac{2}{3}}$, $T_0 = \max \big{\{} 4T^{\frac{2}{3}}, \frac{8}{\alpha} \big{\}}$,  $\epsilon_{m} = 16R^2\beta_{m}^{\frac{1}{3}}$, when $\beta_{m} = \alpha(mK+T_0)$, in Algorithm \ref{alg:Bandit}, guarantees that the expected regret is upper-bounded by
{
\begin{align*}
        \mathbb{E}[\mathcal{R}_{T}] \leq & \left( \frac{4 c R G}{r}  + 2 R^2 \alpha + \frac{4G}{\alpha^{\frac{1}{3}}} \right) T^{\frac{2}{3}} \\
        & + \left( \frac{2nM}{c \sqrt{\alpha}}  +  \frac{2G}{\sqrt{\alpha}} + 4 R \sqrt{\alpha} \right)^2  T^{\frac{2}{3}} (1+\ln(T)),
\end{align*}}
and that the expected overall number of calls to the linear optimization oracle is upper-bounded by
{
\begin{align*}
        \mathbb{E} \left[ \sum_{m=1}^{\frac{T}{K}} L_m \right] \leq &  \left(\frac{ n M}{c 2 R } + \frac{G}{2 R} +  \alpha \right) T \\
        & + \frac{1}{ \alpha^{\frac{2}{3}}} \left(\frac{ n M}{c 2 R} + \frac{G}{2 R} + \alpha \right)^2 T.
\end{align*}}
In particular, if $ \frac{(nM)^2  \ln{T}}{r^2 G R \alpha}  \leq T$ then, setting $c=\left( r(nM)^2 \ln{T} / G R \alpha \right)^\frac{1}{3}$, we have 

{
\begin{align*}
        & \mathbb{E}[\mathcal{R}_{T}] \leq  \left( \frac{4 (n M G R)^\frac{2}{3} (\ln{T})^\frac{1}{3}}{\alpha^\frac{1}{3} r^\frac{2}{3}} + 2 R^2 \alpha + \frac{4G}{\alpha^{\frac{1}{3}}} \right) T^{\frac{2}{3}} \\
         & +   \left( \frac{2 (n M G R)^\frac{1}{3} }{\alpha^\frac{1}{6} r^\frac{1}{3} (\ln{T})^\frac{1}{3}}  +  \frac{2G}{\sqrt{\alpha}} + 4 R \sqrt{\alpha} \right)^2  T^{\frac{2}{3}} (1+\ln(T)),
\end{align*}}
and     
{
\begin{align*}
        \mathbb{E} \left[ \sum_{m=1}^{\frac{T}{K}} L_m \right] \leq &  \left(\frac{  (n M G \alpha)^\frac{1}{3} }{2 r^\frac{1}{3}  R^\frac{2}{3} (\ln{T})^\frac{1}{3} } + \frac{G}{2 R} +  \alpha \right) T \\
        & + \frac{1}{ \alpha^{\frac{2}{3}}} \left(\frac{  (n M G \alpha)^\frac{1}{3} }{2 r^\frac{1}{3}  R^\frac{2}{3}  (\ln{T})^\frac{1}{3} } + \frac{G}{2 R} + \alpha \right)^2 T.
\end{align*}}
\end{theorem}

\subsection{Proof of Theorem \ref{thm:bandit_main_sc}}

The following lemma is a key ingredient for obtaining improved bounds for the bandit setting. At a high-level it can be used to show that the ``drift" in objective $\hat{F}_m(\cdot)$ from one round to the next, which is due to the gradients on each new block, as a first approximation, scales like $\sqrt{K}$ (the second term on the RHS $K^2G^2$ will typically be smaller) and not linear in $K$. On the other-hand, during such a block of length $K$ we can squeeze $O(K)$ linear optimization steps. This intuitively explains why the optimal tradeoff is to take block-length $K>1$ (as opposed to the full-information setting).

\begin{lemma} \label{lemma:expectation_gradient}[ Lemma 5 in  \cite{garber2020improved}]
    For any iteration (block) $m$ of the outer-loop in Algorithm \ref{alg:Bandit} it holds that
    \begin{align}
        \mathbb{E} \left[  \Vert \hat{\mathbf{g}}_m \Vert  \right]^2 \leq \mathbb{E} \left[  \Vert \hat{\mathbf{g}}_m \Vert ^2 \right] \leq K (nm\delta^{-1})^2 + K^2 G^2. \nonumber
    \end{align}
\end{lemma}

The following lemma is analogous to Lemma \ref{lemma:full_regret} in the full-information setting, and upper-bounds the regret of Algorithm \ref{alg:Bandit} in terms of the approximation error each iterate $\x_m$ guarantees with respect to the corresponding objective $\hat{F}_m(\cdot)$ (which is captured by the scalar $\epsilon_m$). The proof is given in the appendix.

\begin{lemma}\label{lemma:Bandit_RFTL_sc}
Let $\{\epsilon_{m}\}_{m=1}^{T/K} > 0$. Suppose that throughout the run of Algorithm \ref{alg:Bandit}, for all blocks $m = 1, \dots, \frac{T}{K}$ it holds that $\hat{F}_{m}(\mathbf{x}_{m}) -  \hat{F}_{m}(\mathbf{x}_{m}^*) \leq \epsilon_{m}$. Then, the expected regret of the algorithm is upper-bounded by
{
    \begin{align*}
        \mathbb{E}[\mathcal{R}_{T}] \leq & \frac{4}{\alpha} \left(  \frac{nM}{\delta} + \sqrt{K} (G + 2 R \alpha) \right)^2 (1 + \ln{T}) \\
        & + 2 \alpha R^2 T_0 + 3 \delta G T + \delta \frac{R}{r} G T  \\
        & + \sqrt{\frac{2}{\alpha}} G K \sum_{m=1}^{\frac{T}{K}} \sqrt{\frac{\epsilon_{m}}{(m-1) K + T_0}}.
    \end{align*}}
\end{lemma}

The following lemma is analogues to the use of Lemma \ref{lemma:epsilon_gap} in the full-information setting, and
is used to upper-bound the number of iterations required by the Frank-Wolfe method, Algorithm \ref{alg:Cg}, to terminate on each invocation from  Algorithm \ref{alg:Bandit}. A proof is given in the appendix.

\begin{lemma}\label{lemma:cg_epsilon_error_L_iterations}
     Let $\epsilon > 0$. Given a function $F(\mathbf{x})$, $2\beta$-smooth, and $\mathbf{x}_{1} \in \mathcal{K}_{\delta}$ such that $h_1 := F(\mathbf{x}_{1}) - F(\mathbf{x}^*)$, where $\mathbf{x}^* = \argmin\limits_{\mathbf{x} \in \mathcal{K}_{\delta}}F(\mathbf{x}) $, Algorithm \ref{alg:Cg} produces a point $\mathbf{x}_{L+1}\in\mK_{\delta}$ such that $F(\mathbf{x}_{L+1}) - F(\mathbf{x}^*) \leq \epsilon$ after at most $L = \max \bigg{\{} \frac{4 \beta (2R)^2}{\epsilon^2} (h_1 - \epsilon) , 
    ~ \frac{2}{\epsilon} (h_1 - \epsilon) \bigg{\}}$ iterations.
\end{lemma}


\begin{proof}[Proof of Theorem \ref{thm:bandit_main_sc}]
We first upper bound the expected overall number of calls to the linear optimization oracle throughout the run of the algorithm, and then we upper-bound the expected regret.

Let $\mathbf{z}_{m,\tau}$ be the iterate of Algorithm \ref{alg:Cg} after completing $\tau-1$ iterations of the do-while loop, when invoked on iteration (block) $m$ of Algorithm \ref{alg:Bandit}. Also, for all $m,\tau$, define $h_{m,\tau} := \hat{F}_{m}(\mathbf{z}_{m,\tau}) -  \hat{F}_{m}(\mathbf{x}_{m}^*)$. Recall that for any iteration $m$ of Algorithm \ref{alg:Cg}, we have $\mathbf{z}_{m,1} = \mathbf{x}_{m-1}$.

Using the fact $ \hat{F}_{m}(\x_{m+1}^*) \geq \hat{F}_{m}(\x_{m}^*)$, we have
\begin{align*}
    h_{m+1,1} & =  \hat{F}_{m+1}(\mathbf{z}_{m+1,1}) -  \hat{F}_{m+1}(\x_{m+1}^*)  \nonumber \\
    & \leq  \hat{F}_{m}(\x_{m}) -  \hat{F}_{m}(\x_{m}^*) + g_m( \x_{m} ) - g_m( \x_{m+1}^* ) \nonumber \\
      & \leq \epsilon_m +\left( \Vert \hat{\mathbf{g}}_m ^\top \Vert + 2 R K \alpha \right)  \Vert \x_{m} - \x_{m+1}^* \Vert .
\end{align*}
The last inequality is since, $h_{m,L_{m}} = F_{m}(\x_{m}) - F_{m}(\x_{m}^*) \leq \epsilon_m$ and Eq.\eqref{eq:Bandit_Lip_g_tilde_alpha_strongly}.

Using triangle inequality, we have
\begin{align*}
     \mathbb{E}[h_{m+1,1}] \leq & \epsilon_{m} + \mathbb{E} \left[ \left(\Vert \hat{\mathbf{g}}_m \Vert  + 2 R K \alpha \right) \Vert \mathbf{x}_{m} - \mathbf{x}_{m}^* \Vert  \right] \\
     & + \mathbb{E} \left[ \left(\Vert \hat{\mathbf{g}}_m \Vert  + 2 R K \alpha \right)  \Vert \mathbf{x}_{m}^* - \mathbf{x}_{m+1}^* \Vert \right].
\end{align*}

Since $h_{m,L_{m}} \leq \epsilon_{m}$, using Eq. \eqref{eq:strong_convexity}, we have that $ \Vert \mathbf{x}_{m} - \mathbf{x}_{m}^* \Vert \leq \sqrt{\frac{2 \epsilon_{m} }{((m-1) K + T_0)\alpha}}  $. Also, from Eq. \eqref{eq:bandit_absolut_dist_sc}, we have $ \Vert \mathbf{x}_{m}^* - \mathbf{x}_{m+1}^* \Vert \leq \frac{2}{(m K + T_0)\alpha} \left( \Vert \hat{\mathbf{g}}_m \Vert   + 2 R K \alpha \right) $. Thus, we have
\begin{align}
    \mathbb{E}[h_{m+1,1}] \leq &  \epsilon_{m} + \frac{\sqrt{2 \epsilon_{m}}\left(\mathbb{E} \left[ \Vert \hat{\mathbf{g}}_m \Vert  \right] + 2 R K \alpha \right)}{\sqrt{\alpha((m-1) K + T_0)}}  \nonumber \\
    & + \frac{2 \mathbb{E} \left[  \Vert \hat{\mathbf{g}}_m \Vert   + 2 R K \alpha  \right]^2}{\alpha(m K + T_0)} . \label{eq:Bandit_upper_bound_h_1_cg_lemma_sc}
\end{align}

Using Lemma \ref{lemma:cg_epsilon_error_L_iterations} with $h_1 = h_{m,1}$, then we have
{
\begin{align*}
    \mathbb{E}[h_1] \leq & \epsilon_{m-1} + \frac{\sqrt{2 \epsilon_{m-1}} \left(\mathbb{E} \left[ \Vert \hat{\mathbf{g}}_{m-1} \Vert  \right] + 2 R K \alpha \right)}{\sqrt{\alpha\left((m-2) K + T_0\right)}} \\
    & + \frac{2 \mathbb{E} \left[  \Vert \hat{\mathbf{g}}_{m-1} \Vert   + 2 R K \alpha  \right]^2 }{\alpha\left((m-1) K + T_0\right)}. 
\end{align*}}

for $m = 1, \dots, \frac{T}{K}$, we have that on each iteration (block) $m$, the number of calls to the linear optimization oracle is $L_m \leq \max \bigg{\{} \frac{2 \beta_{m} (2R)^2}{\epsilon_{m}^2} (h_{m,1} - \epsilon_{m}) ,~ \frac{2}{\epsilon_{m}} (h_{m,1} - \epsilon_{m}) \bigg{\}}$, when $\hat{F}_{m}(\x)$ is $2\beta_m$-Smooth, note that $\beta_m = \alpha(mK+T_0)$. Since $\epsilon_{m} = 16R^2\beta_{m}^{\frac{1}{3}}$ and $T_0 \geq \frac{8}{\alpha}$ for any $m$ we have $\beta_{m} (2R)^2 > \epsilon_{m}$, then we have $ L_{m}\leq \frac{2 \beta_{m} (2R)^2}{\epsilon_{m}^2} (h_{m,1} - \epsilon_{m})$.
 Following Eq. \eqref{eq:Bandit_upper_bound_h_1_cg_lemma_sc}, Lemma \ref{lemma:expectation_gradient} and plugging in $\epsilon_{m}$, we have 
{
\begin{align*}
    \mathbb{E}[L_{m}]  \leq & \frac{1}{8 R^2}  \frac{ \beta_{m}^\frac{1}{3} }{4 } \left( \mathbb{E}[h_{m,1}] - \epsilon_{m} \right)  \\
    \leq & \frac{\sqrt{2}}{8 R}  \frac{ \beta_{m}^\frac{1}{2} }{\beta_{m-2}^\frac{1}{2}} \left(\frac{\sqrt{K} n M}{\delta} + K G + 2 R K \alpha \right) \\
    &  + \frac{  \beta_{m}^\frac{1}{3} }{16 R^2 \beta_{m-1}}  \left(\frac{\sqrt{K} n M}{\delta} + K G + 2 R K \alpha \right)^2. 
\end{align*}}

Since for every $m = 0, \dots, \frac{T}{K}$, $\beta_{m} = \frac{\alpha (mK +T_0)}{2}$ and $T_0 \geq 4K$, it holds that $\sum_{m=1}^{\frac{T}{K}} \left( \beta_{m}^{\frac{1}{2}} / \beta_{m-2}^{\frac{1}{2}} \right) \leq \sqrt{2} \frac{T}{K}$ and $ \sum_{m=2}^{\frac{T}{K}} \frac{\left(\beta_{m}\right)^{\frac{1} {3}}} {\beta_{m-1}} \leq \frac{4}{\alpha^\frac{2}{3}} \frac{ T^{\frac{1}{3}}}{K}$. Then, overall on all blocks, we obtain 
{
\begin{align*}
    & \mathbb{E} \left[ \sum_{m=1}^{\frac{T}{K}} L_m \right] \leq \frac{1}{4 R} \frac{T}{ K} \left(\frac{\sqrt{K} n M}{\delta} + K G + 2 R K \alpha \right) \nonumber \\
    & ~~~~~~~~~~~~ +  \frac{1}{4 R^2 \alpha^{\frac{2}{3}}} \frac{ T^{\frac{1}{3}}}{K}  \left(\frac{\sqrt{K} n M}{\delta} + K G + 2 R K \alpha \right)^2    \nonumber \\
    & \underset{(d)}{\leq} ~ \left(\frac{ n M}{c 4 R } + \frac{G}{4 R} +  \frac{\alpha}{2} \right) T + \frac{1}{ \alpha^{\frac{2}{3}}} \left(\frac{ n M}{c 2 R} + \frac{G}{2 R} + \alpha \right)^2 T.
\end{align*}}
Inequality (d) is due to plugging-in $\delta, K$.

We now turn to upper-bound the expected regret of the algorithm. using Lemma \ref{lemma:Bandit_RFTL_sc}, plugging-in $\epsilon_m$ and the fact that $\sum_{m=1}^{\frac{T}{K}} \sqrt{\frac{ \left(\beta_{m}\right)^{\frac{1}{3}}}{\beta_{m-1}}} \leq \frac{2}{\alpha^{\frac{1}{3}}} \frac{T^{\frac{2}{3}}}{K}$, we obtain
{
\begin{align*}
    \mathbb{E}[\mathcal{R}_{T}] \leq & 3 \delta G T + \delta \frac{R}{r} G T  + 2 \alpha R^2 T_0  +  \frac{2\sqrt{2}}{\alpha^{\frac{1}{3}}} G T^{\frac{2}{3}} \nonumber \\
    & + \frac{4}{\alpha} \left(  \frac{nM}{\delta} + \sqrt{K} (G + 2 R \alpha) \right)^2 (1 + \ln{(T)}) \nonumber \\
    \underset{(d)}{\leq} & \left( 3 c G + c \frac{R}{r} G + 2 R^2 \alpha + \frac{2\sqrt{2}G}{\alpha^{\frac{1}{3}}} \right) T^{\frac{2}{3}} \nonumber \\ 
    & + \frac{4}{\alpha} \left( \frac{nM}{c }  +  G + 2 R \alpha \right)^2  T^{\frac{2}{3}} (1 + \ln{(T)}).
\end{align*}}
Equality (d) is due to plugging-in $\delta, K, T_0$.

\end{proof}

\section{DISCUSSION}
We have proved strong convexity leads to faster rates for projection-free online learning both in the full-information and bandit settings.  This is obtained via the standard Online Frank-Wolfe method \cite{Hazan12} and its bandit variant with blocks \cite{garber2020improved}. In particular, this improvement is in stark contrast to the state-of-affairs in offline convex optimizaiton, in which, in general, strong convexity does not lead to faster rates for the Frank-Wolfe method. 

In light of this current work and the recent work, \cite{HazanM20}, it is interesting if the combination of smoothness and strong convexity can lead to a faster rate than $O(T^{2/3})$. Also, given the recent interest in projection-free online learning, and in particular with a linear optimization oracle, it would be interesting to derive lower bounds on the regret of such algorithms, perhaps with online variants of standard constructions used in the offline setting (see for instance  \cite{Jaggi13,  GH16}).

\section{Acknowledgments}
This research was supported by the ISRAEL SCIENCE FOUNDATION (grant No. 1108/18).

\bibliography{bib}
\bibliographystyle{plain}
\newpage

\onecolumn
\appendix

\section{MISSING PROOFS}

 \subsection{Proof of Lemma \ref{lemma:FTL_BTL}}
 We first restate the lemma and then prove it.\\
 \emph{Lemma} \ref{lemma:FTL_BTL}.
 Let $\mK \subseteq \reals^d$ a convex and compact set, $\{ g_m (\x)  \}_{m=1}^{T}$ a set of convex functions, $\x_1 \in \reals^d$ and $c_1 \in \reals^+$. Denote $\x_{\tau}^* = \argmin\limits_{\x \in \mK} \Big{\{}  \sum_{m=1}^{\tau-1} g_m (\x) +  c_1 \Vert \x - \x_1 \Vert^2 \Big{\}} $  for every $\tau \in [1, T+1]$. Then for every $\x \in \mK$ we have that
 \begin{align*}
     \sum_{m=1}^{T} \big( g_m( \x_{m}^*) - g_m( \x) \big) \leq & \sum_{m=1}^{T} \big(g_m( \x_{m}^*) - g_m( \x_{m+1}^*) \big)  + c_1 \Vert \x - \x_1 \Vert^2.
 \end{align*}

 \begin{proof}
 Define $g_0(\x) = c_1 \Vert \x - \x_1 \Vert^2$. We will show by induction on $\tau$ that for all $\x \in \mK$ the following holds
 \begin{align}
     \sum_{m=0}^{\tau} g_m( \x_{m+1}^*) \leq \sum_{m=0}^{\tau} g_m( \x).
     \label{eq:FTL_BTL}
 \end{align}
 \emph{Induction base}: for $\tau =0$, since $\x_1^* = \x_1$, we have $ g_0( \x_{1}^*) = 0 \leq  c_1 \Vert \x -\x_1 \Vert^2 = g_0( \x)$.\\
 \emph{Induction assumption}: for $\tau = t$ Eq.\eqref{eq:FTL_BTL} holds. \\
 \emph{Induction step}: for $\tau = t+1$ we have that
 \begin{align*}
     \sum_{m=0}^{t+1} g_m( \x_{m+1}^*)  = \sum_{m=0}^{t} g_m( \x_{m+1}^*) + g_{t+1}( \x_{t+2}^*)  & \underset{(i)}{\leq} \sum_{m=0}^{t} g_m( \x_{t+2}^*) + g_{t+1}( \x_{t+2}^*) \leq  \sum_{m=0}^{t+1} g_m( \x_{t+2}^*)   \underset{(ii)}{\leq} \sum_{m=0}^{t+1} g_m( \x),
 \end{align*}
 where inequality (i) is from the induction assumption and inequality (ii) is by definition of $\x_{t+2}^*$.\\
 From Eq.\eqref{eq:FTL_BTL} and the fact that $\x_{1}^{*} = \x_1$ we obtain the inequality.
 \end{proof}

\subsection{Proof of Lemma \ref{lemma:Bandit_RFTL_sc}}

We first restate the lemma and then prove it.\\
\emph{Lemma} \ref{lemma:Bandit_RFTL_sc}.
Let $\{\epsilon_{m}\}_{m=1}^{T/K} > 0$. Suppose that throughout the run of Algorithm \ref{alg:Bandit}, for all blocks $m = 1, \dots, \frac{T}{K}$ it holds that $\hat{F}_{m}(\mathbf{x}_{m}) -  \hat{F}_{m}(\mathbf{x}_{m}^*) \leq \epsilon_{m}$. Then, the expected regret of the algorithm is upper-bounded by
    \begin{align*}
        \mathbb{E}[\mathcal{R}_{T}] \leq & \frac{4}{\alpha} \left(  \frac{nM}{\delta} + \sqrt{K} (G + 2 R \alpha) \right)^2 (1 + \ln{(T)}) + 2 \alpha R^2 T_0 + 3 \delta G T + \delta \frac{R}{r} G T + \sqrt{\frac{2}{\alpha}} G K \sum_{m=1}^{\frac{T}{K}} \sqrt{\frac{\epsilon_{m}}{(m-1) K + T_0}}.
    \end{align*}

\begin{proof}
 It holds that 
\begin{align}
        \mathbb{E}\left[\mathcal{R}_{T}\right] = & \sum_{t=1}^{T} \mathbb{E} \left[f_t(\mathbf{y}_{t})\right]  - \sum_{t=1}^{T} f_t(\mathbf{x}^*) \nonumber \\
        = & \sum_{t=1}^{T} \mathbb{E} \left[f_t(\mathbf{y}_{t})\right] - \sum_{t=1}^{T} \mathbb{E} \left[ \hat{f}_{t,\delta}(\mathbf{x}_{m(t)-1}) \right]   + \sum_{t=1}^{T} \mathbb{E} \left[\hat{f}_{t,\delta}(\mathbf{x}_{m(t)-1}) \right] - \sum_{t=1}^{T} \hat{f}_{t,\delta} (\tilde{\mathbf{x}}^*)  + \sum_{t=1}^{T} \hat{f}_{t,\delta} (\tilde{\mathbf{x}}^*) - \sum_{t=1}^{T} f_t(\mathbf{x}^*). \label{eq:Bandit_full_regret_rftl_lemma_sc}
\end{align}

Since $f_t$ is convex and $G$-Lipschitz, using lemma \ref{lemma:hazan_smooth} and Cauchy-Schwarz inequality, we have that
\begin{align}
    \sum_{t=1}^{T} \mathbb{E} \left[f_t(\mathbf{y}_{t}) - \hat{f}_{t,\delta}(\mathbf{x}_{m(t)-1}) \right]  & = \sum_{t=1}^{T} \mathbb{E} \left[f_t(\mathbf{x}_{m(t)-1} + \delta \mathbf{u}_{t}) - f_t(\mathbf{x}_{m(t)-1}) \right] + \delta G T  \nonumber \\
& \leq \sum_{t=1}^{T} \mathbb{E} \left[G ~  \Vert \delta \mathbf{u}_{t} \Vert \right] + \delta G T \leq 2 \delta G T. \label{eq:Bandit_regret_rftl_1}
\end{align}
Also, we have
\begin{align}
    \sum_{t=1}^{T} \hat{f}_{t,\delta} (\tilde{\mathbf{x}}^*) - f_t(\mathbf{x}^*)   &  = \sum_{t=1}^{T} \hat{f}_{t,\delta} (\tilde{\mathbf{x}}^*) - f_t(\tilde{\mathbf{x}}^*) + f_t(\tilde{\mathbf{x}}^*) - f_t(\mathbf{x}^*)  \leq \delta G T + \sum_{t=1}^{T} G  \Vert  \tilde{\mathbf{x}}^* - \mathbf{x}^*  \Vert \leq \delta G T + \delta \frac{R}{r} G T. \label{eq:Bandit_regret_rftl_2}
\end{align}

Now we need to obtain an upper bound on $\sum_{t=1}^{T} \mathbb{E} \left[ \hat{f}_{t,\delta}(\mathbf{x}_{m(t)-1}) \right] - \sum_{t=1}^{T} \hat{f}_{t,\delta} (\tilde{\mathbf{x}}^*)$. We will first take a few preliminary steps before we get this bound. Define $g_m (\mathbf{x}) = \hat{\mathbf{g}}_m ^\top \mathbf{x} + \frac{K \alpha}{2} \Vert \mathbf{x} - \mathbf{x}_{m-1} \Vert^2$, using Cauchy-Schwarz inequality and Eq.\eqref{eq:gap_of_square_norm}, for all $ \x,\y \in \mK$ we have
\begin{align}
     g_m (\x) - g_m (\y) & \leq  \left( \Vert \hat{\mathbf{g}}_m ^\top \Vert + 2 R K \alpha \right) \Vert \x - \y\Vert. \label{eq:Bandit_Lip_g_tilde_alpha_strongly}
\end{align} 

Since for every $m$, $\hat{F}_{m+1}(\x)$ is $(mK+T_0)\alpha$-strongly convex, using Eq. \eqref{eq:strong_convexity} we have that
\begin{align*}
    \frac{(m K + T_0)\alpha}{2}  \Vert \x_{m}^* -  \x_{m+1}^*\Vert ^2  \leq \hat{F}_{m+1} ( \x_{m}^* ) - \hat{F}_{m+1}(\x_{m+1}^*) & =  \hat{F}_{m} ( \x_{m}^* ) -  \hat{F}_{m}(\x_{m+1}^*) + g_m( \x_{m}^* ) - g_m( \x_{m+1}^* ) \nonumber \\
    & \leq  g_m( \x_{m}^* ) - g_m( \x_{m+1}^* ) .
\end{align*}
The last inequality is since for every $m$,  $\hat{F}_{m}(\x_{m}^*) \leq \hat{F}_{m}(\x_{m+1}^*)$.

From the two last equations, we obtain
\begin{align}
    \Vert\mathbf{x}_{m}^*-\mathbf{x}_{m+1}^*\Vert & \leq 2 \frac{ \Vert \hat{\mathbf{g}}_m ^\top \Vert + 2 R K \alpha }{(m K + T_0)\alpha} . \label{eq:bandit_absolut_dist_sc}
\end{align}

Using Eq. \eqref{eq:Bandit_Lip_g_tilde_alpha_strongly} and Eq. \eqref{eq:bandit_absolut_dist_sc}, we have that
\begin{align}
    \mathbb{E} \left[ \sum_{m=1}^{\frac{T}{K}} g_m( \mathbf{x}_{m}^*) - g_m( \x_{m+1}^*) \right] & \leq  \sum_{m=1}^{\frac{T}{K}} \frac{2}{(m K + T_0)\alpha} \mathbb{E}  \left[  \left( \Vert \hat{\mathbf{g}}_m ^\top \Vert + 2 R K \alpha \right)^2  \right] \nonumber \\
&  \leq \frac{2}{\alpha} \left( \frac{nM}{\delta} + \sqrt{K} (G + 2 R \alpha) \right)^2 (1 + \ln{(T)}) \label{eq:Bandit_Lip_g_tilde_alpha_strongly_1} .
\end{align}

The last inequality is due to Lemma \ref{lemma:expectation_gradient} and the fact that for all $a,b \in \reals^+$ it holds that $\sqrt{a+b} \leq \sqrt{a} + \sqrt{b}$, and also the fact, for any $T_0 \geq 0$, it holds that $ \sum_{m=1}^{ \frac{T}{K} }  \left( \frac{1}{m K + T_0} \right) \leq \frac{1 + \ln{(T)}}{K}$.

Using the $\alpha$-strong convexity of $f_t$ and Lemma \ref{lemma:hazan_smooth}, it holds that $\hat{f}_{t,\delta}$ is $\alpha$-strongly convex. Then, $\hat{f}_{t,\delta}(\mathbf{x}) -  \hat{f}_{t,\delta}(\mathbf{y}) \leq \nabla \hat{f}_{t,\delta}^{\top}(\mathbf{x}) (\mathbf{x} - \mathbf{y}) - \frac{\alpha}{2} \Vert \x - \y \Vert^2$, 
and we have
\begin{align}
     \sum_{t=1}^{T}   \Big( \mathbb{E} &  \left[  \hat{f}_{t,\delta}(\mathbf{x}_{m(t)-1}) \right] -  \hat{f}_{t,\delta}(\tilde{\mathbf{x}}^*) \Big)   \leq \sum_{t=1}^{T} \mathbb{E} \left[ \hat{\nabla}_{t,\delta,m(t)-1}^{\top} (\mathbf{x}_{m(t)-1} - \tilde{\mathbf{x}}^*) - \frac{\alpha}{2} \Vert \mathbf{x}_{m(t)-1} - \tilde{\mathbf{x}}^* \Vert^2 \right] \nonumber \\
    & ~~~~~~ =  \sum_{t=1}^{T} \mathbb{E} \Big[ \hat{\nabla}_{t,\delta,m(t)-1}^{\top} \left( \mathbf{x}_{m(t)-1} - \mathbf{x}_{m(t)-1}^* + \mathbf{x}_{m(t)-1}^* - \mathbf{x}_{m(t)}^*  + \mathbf{x}_{m(t)}^* - \tilde{\mathbf{x}}^* \right) - \frac{\alpha}{2} \Vert \mathbf{x}_{m(t)-1} - \tilde{\mathbf{x}}^* \Vert^2 \Big]  \label{eq:Bandit_regret_rftl_delta_smooth}
\end{align}

Using Lemma \ref{lemma:bertsekas_grdient} we have that for all $t\in[T]$, $\Vert{\nabla {\hat{f}}_{t,\delta} (\mathbf{x}_t)}\Vert \leq G$.
Since for all $m$, $\hat{F}_{m+1}(\mathbf{x})$ is $(m K+T_0)\alpha$-strongly convex, using our assumption that $\hat{F}_{m}(\mathbf{x}_{m}) -  \hat{F}_{m}(\mathbf{x}_{m}^*) \leq \epsilon_{m}$, we have that 
\begin{align}
     \sum_{t=1}^{T} \mathbb{E} \left[ \nabla \hat{f}_{t,\delta}(\mathbf{x}_{m-1})^{\top}  (\mathbf{x}_{m-1} - \mathbf{x}_{m-1}^*) \right] & \leq \sqrt{\frac{2}{\alpha}} G K \sum_{m=1}^{\frac{T}{K}}  \mathbb{E} \left[ \sqrt{\frac{\hat{F}_{m-1}(\mathbf{x}_{m-1}) - \hat{F}_{m-1}(\mathbf{x}_{m-1}^*)}{(m-2) K + T_0}} \right] \nonumber \\
    & \leq \sqrt{\frac{2}{\alpha}} G K \sum_{m=1}^{\frac{T}{K}} \sqrt{\frac{\epsilon_{m-1}}{(m-2) K + T_0}}. \label{eq:Bandit_cg_regret_rftl_lemma_sc}
\end{align}

Using Eq.\eqref{eq:bandit_absolut_dist_sc} and the fact for any $T_0 \geq 0$ it holds that $ \sum_{m=1}^{ \frac{T}{K} }  \left( \frac{1}{m K + T_0} \right) \leq \frac{1 + \ln{(T)}}{K}$, we have
\begin{align}
    \sum_{t=1}^{T} \mathbb{E} \left[ \nabla \hat{f}_{t,\delta}(\mathbf{x}_{m(t)-1})^{\top} (\mathbf{x}_{m(t)-1}^* - \mathbf{x}_{m(t)}^*) \right] & \leq GK \sum_{m=1}^{\frac{T}{K}} \mathbb{E} \left[ \Vert (\mathbf{x}_{m-1}^* - \mathbf{x}_{m}^*) \Vert \right] \nonumber \\
    & \leq \frac{2G}{\alpha} \left(\mathbb{E} \left[ \Vert \hat{\mathbf{g}}_{m-1} ^\top \Vert  \right] + 2 R K \alpha \right) (1 + \ln{(T)})  \nonumber \\
    & \leq  \frac{2 G}{\alpha} \left( \sqrt{K} \frac{n M }{ \delta} + K \left(G  + 2 R \alpha \right) \right) (1 + \ln{(T)}) . \label{eq:Bandit_osb_regret_rftl_lemma_sc}
\end{align}
The last inequality is due to Lemma \ref{lemma:expectation_gradient} and the fact that for all $a,b \in \reals^+$ it holds that $\sqrt{a+b} \leq \sqrt{a} + \sqrt{b}$.

Now, it remains to obtain an upper bound on $\sum_{t=1}^{T} \mathbb{E}  \Big[  \hat{\nabla}_{t,\delta,m(t)-1}^{\top}  (\mathbf{x}_{m(t)}^* - \tilde{\mathbf{x}}^*) - \frac{\alpha}{2} \Vert \mathbf{x}_{m(t)-1} - \tilde{\mathbf{x}}^* \Vert^2 \Big]$.

Define for all $m \in \left[\frac{T}{K}\right]$ $\mathcal{F}_m = \{ \mathbf{x}_1, \hat{\mathbf{g}}_1, \dots, \mathbf{x}_{m-1}, \hat{\mathbf{g}}_{m-1} \} $- the history of all predictions and gradient estimates. Throughout the sequel we introduce the short notation $ \hat{\nabla}_{t,\delta,m(t)-1} = \nabla {\hat{f}}_{t,\delta}(\mathbf{x}_{m(t)-1})$.
Using Lemma \ref{lemma:hazan_gradient}, it holds that $\mathbf{g}_t$ is an unbiased estimator of $\hat{\nabla}_{t,\delta,m(t)-1}$, then $\mathbb{E} \left[ \mathbf{g}_t | \mathcal{F}_m \right] = \hat{\nabla}_{t,\delta,m(t)-1}$. Since  $\mathbf{x}_m^* = \argmin\limits_{\mathbf{x} \in \mathcal{K}_{\delta}} \Big{\{} \hat{F}_{m}(\mathbf{x}) := \sum_{i=1}^{m-1}  g_i(\x)  + T_0 \frac{\alpha}{2} \Vert\mathbf{x} - \mathbf{x}_1\Vert^2 \Big{\}} $, when $ \{ g_i (\mathbf{x}) = \hat{\mathbf{g}}_i ^\top \mathbf{x} + \frac{K \alpha}{2} \Vert \mathbf{x} - \mathbf{x}_{i-1} \Vert^2 \}_{i=1}^{m}$, we have that $\mathbb{E} \left[ \mathbf{x}_m^*| \mathcal{F}_m \right] = \mathbf{x}_m^*$. From both observations $\forall \mathbf{x} \in \mathcal{K}_{\delta}$ and $\forall m \in \left[\frac{T}{K}\right] $, it holds that
\begin{align*}
    \mathbb{E}   \left[ \hat{\mathbf{g}}_m^{\top}  (\mathbf{x}_{m}^* - \mathbf{x}) \right] =  \mathbb{E} \left[ \mathbb{E} \left[ \hat{\mathbf{g}}_m| \mathcal{F}_{m} \right] ^{\top}  (\mathbf{x}_{m}^* - \mathbf{x}) \right] & = \mathbb{E} \left[ \sum_{t=(m-1)K+1}^{mK} \mathbb{E} \left[  \mathbf{g}_t| \mathcal{F}_{m(t)} \right] ^{\top}  (\mathbf{x}_{m(t)}^* - \mathbf{x}) \right] \nonumber \\
    & = \sum_{t=(m-1)K+1}^{mK}  \mathbb{E} \left[ \hat{\nabla}_{t,\delta,m(t)-1}^{\top} (\mathbf{x}_{m(t)}^* - \mathbf{x}) \right]. 
\end{align*}

Then, from the definition of $g_m(\x)$, and the fact that $\frac{\alpha}{2} \Vert \mathbf{x}_{m-1} - \x_m^* \Vert^2 \geq 0$, we have 
\begin{align}
    \sum_{t=1}^{T} \mathbb{E}  \Big[  \hat{\nabla}_{t,\delta,m(t)-1}^{\top}  (\mathbf{x}_{m(t)}^* - \tilde{\mathbf{x}}^*) - \frac{\alpha}{2} \Vert \mathbf{x}_{m(t)-1} - \tilde{\mathbf{x}}^* \Vert^2 \Big] & \leq  \mathbb{E} \left[ \sum_{m=1}^{\frac{T}{K}} g_m( \mathbf{x}_{m}^*) - g_m( \tilde{\x}^*) \right] \label{eq:Bandit_first_observation}.
\end{align}

Using Lemma \ref{lemma:FTL_BTL}, for all $\mathbf{x} \in \mathcal{K}_{\delta}$, we have that
\begin{align*}
    \sum_{m=1}^{\frac{T}{K}} g_m( \mathbf{x}_{m}^*) - g_m( \mathbf{x}) \leq & \sum_{m=1}^{\frac{T}{K}} g_m( \mathbf{x}_{m}^*) - g_m( \mathbf{x}_{m+1}^*)  + \frac{T_0 \alpha}{ 2} \Vert\mathbf{x} - \mathbf{x}_1\Vert^2. 
\end{align*}

Then, from Eq.\eqref{eq:Bandit_first_observation}, we have 
\begin{align}
     \sum_{t=1}^{T} \mathbb{E}  \Big[  \hat{\nabla}_{t,\delta,m(t)-1}^{\top}  (\mathbf{x}_{m(t)}^* - \tilde{\mathbf{x}}^*) - \frac{\alpha}{2} \Vert \mathbf{x}_{m(t)-1} - \tilde{\mathbf{x}}^* \Vert^2 \Big] & \leq \mathbb{E} \left[ \sum_{m=1}^{\frac{T}{K}} g_m( \mathbf{x}_{m}^*) - g_m( \x_{m+1}^*) \right] + ~ \frac{T_0 \alpha}{ 2 } \Vert\tilde{\mathbf{x}}^* - \mathbf{x}_1\Vert^2 \nonumber \\
    & \leq \frac{2}{\alpha} \left( \frac{nM}{\delta}+ \sqrt{K} (G + 2 R \alpha) \right)^2 (1 + \ln{(T)}) +  2 \alpha R^2 T_0. \label{eq:Bandit_rftl_regret_rftl_lemma_sc}
\end{align}
The last inequality is from Eq.\eqref{eq:Bandit_Lip_g_tilde_alpha_strongly_1}.
Combining the results of Eq. \eqref{eq:Bandit_rftl_regret_rftl_lemma_sc},  \eqref{eq:Bandit_cg_regret_rftl_lemma_sc} and  \eqref{eq:Bandit_osb_regret_rftl_lemma_sc} in Eq.\eqref{eq:Bandit_regret_rftl_delta_smooth}, we obtain
 \begin{align*}
    \sum_{t=1}^{T}   \Big( \mathbb{E}  \left[  \hat{f}_{t,\delta}(\mathbf{x}_{m(t)-1}) \right]  -  \hat{f}_{t,\delta}(\tilde{\mathbf{x}}^*) \Big) \leq &  \frac{4}{\alpha} \left(  \frac{nM}{\delta} + \sqrt{K} (G + 2 R \alpha) \right)^2 (1 + \ln{(T)})+  2 \alpha R^2 T_0 \nonumber \\
& +  \sqrt{\frac{2}{\alpha}} G K \sum_{m=1}^{\frac{T}{K}} \sqrt{\frac{\epsilon_{m}}{(m-1) K + T_0}} . 
\end{align*}

Combining the last result with Eq. \eqref{eq:Bandit_regret_rftl_1} and Eq.\eqref{eq:Bandit_regret_rftl_2} in Eq.\eqref{eq:Bandit_full_regret_rftl_lemma_sc}, we obtain the required bound.
\end{proof}

\subsection{Proof of Lemma \ref{lemma:cg_epsilon_error_L_iterations}}

We first restate the lemma and then prove it.\\
\emph{Lemma} \ref{lemma:cg_epsilon_error_L_iterations}.
 Let $\epsilon > 0$. Given a function $F(\mathbf{x})$, $2\beta$-smooth, and $\mathbf{x}_{1} \in \mathcal{K}_{\delta}$ such that $h_1 := F(\mathbf{x}_{1}) - F(\mathbf{x}^*)$, where $\mathbf{x}^* = \argmin\limits_{\mathbf{x} \in \mathcal{K}_{\delta}}F(\mathbf{x}) $, Algorithm \ref{alg:Cg} produces a point $\mathbf{x}_{L+1}\in\mK_{\delta}$ such that $F(\mathbf{x}_{L+1}) - F(\mathbf{x}^*) \leq \epsilon$ after at most $L = \max \bigg{\{} \frac{4 \beta (2R)^2}{\epsilon^2} (h_1 - \epsilon) , 
    ~ \frac{2}{\epsilon} (h_1 - \epsilon) \bigg{\}}$ iterations.

\begin{proof}
    For any iteration $\tau$ of Algorithm \ref{alg:Cg}, define $h_{\tau} :=  F(\mathbf{x}_{\tau}) - F(\mathbf{x}^*)$ and denote $\nabla_{\tau} := \nabla F(\mathbf{x}_{\tau})$. From the choice of $\v_{\tau}$ and the convexity of $F(\cdot)$, it follows that
    \begin{align}
         h_\tau = F(\mathbf{x}_{\tau}) - F(\mathbf{x}^*)  \leq \nabla_{\tau}^{\top} (\mathbf{x}_{\tau} - \mathbf{x}^*) \leq \nabla_{\tau}^{\top} (\mathbf{x}_{\tau} - \mathbf{v}_{\tau}). \label{eq:aux1_cg_lemma}
    \end{align}
    We now establish the convergence rate of Algorithm \ref{alg:Cg}. It holds that
    \begin{align*}
        h_{\tau+1} = F(\mathbf{x}_{\tau+1}) -  F(\mathbf{x}^*)  = F(\mathbf{x}_{\tau} + \sigma_{\tau}(\mathbf{v}_{\tau} - \mathbf{x}_{\tau})) - F(\mathbf{x}^*).
    \end{align*}
    For our analysis we define $\hat{\sigma}_{\tau} = \min \Big{\{} \frac{ \nabla_{\tau}^{\top} (\mathbf{x}_{\tau} - \mathbf{v}_{\tau})}{2 \beta (2R)^2}, 1 \Big{\}}$.
    Since $\sigma_{\tau}$ is chosen via line-search, we have that
    \begin{align*}
        h_{\tau+1} & =  F(\mathbf{x}_{\tau} + \sigma_{\tau}(\mathbf{v}_{\tau} - \mathbf{x}_{\tau})) - F(\mathbf{x}^*)  \leq F(\mathbf{x}_{\tau} + \hat{\sigma}_{\tau}(\mathbf{v}_{\tau} - \mathbf{x}_{\tau})) - F(\mathbf{x}^*). 
    \end{align*}
    Since $F(\mathbf{x})$ is $2\beta$-smooth, it holds that 
    \begin{align*}
        F(\mathbf{x}_{\tau} & + \hat{\sigma}_{\tau}(\mathbf{v}_{\tau} - \mathbf{x}_{\tau})) \leq  F(\mathbf{x}_{\tau}) + \hat{\sigma}_{\tau} \nabla_{\tau}^{\top} (\mathbf{v}_{\tau} - \mathbf{x}_{\tau}) + \beta \hat{\sigma}_{\tau}^2  \Vert \mathbf{v}_{\tau} - \mathbf{x}_{\tau} \Vert ^2, 
    \end{align*}
    and we obtain
    \begin{align}
        h_{\tau+1} \leq & h_{\tau} + \hat{\sigma}_{\tau}^2 \beta (2R)^2 - \hat{\sigma}_{\tau} \nabla_{\tau}^{\top} (\mathbf{x}_{\tau} - \mathbf{v}_{\tau}).   \nonumber
    \end{align}
    We now consider several cases.
    \\ \emph{Case 1}: If $\nabla_{\tau}^{\top} (\mathbf{x}_{\tau} - \mathbf{v}_{\tau}) \leq \epsilon$ for some $\tau < L$, the algorithm will stop after less than $L$ iterations. Moreover, from Eq. \eqref{eq:aux1_cg_lemma} we have $h_{\tau} \leq \epsilon$.
    \\ \emph{Case 2}: Else, $\nabla_{\tau}^{\top} (\mathbf{x}_{\tau} - \mathbf{v}_{\tau}) \geq \epsilon$ for all $\tau < L$. We have two cases: \\
    \emph{Case 2.1}: If $\nabla_{\tau}^{\top} (\mathbf{x}_{\tau} - \mathbf{v}_{\tau}) \geq 2 \beta (2R)^2$ then $\hat{\sigma}_{\tau} = 1$ and we have
    \begin{align*}
        h_{\tau+1} & \leq  h_{\tau} + \hat{\sigma}_{\tau}^2 \beta (2R)^2 - \hat{\sigma}_{\tau} \nabla_{\tau}^{\top} (\mathbf{x}_{\tau} - \mathbf{v}_{\tau}) \leq  h_{\tau} - \frac{ \nabla_{\tau}^{\top} (\mathbf{x}_{\tau} - \mathbf{v}_{\tau})}{2} .
    \end{align*}
        \\ \emph{Case 2.2}: Else, $ \nabla_{\tau}^{\top} (\mathbf{x}_{\tau} - \mathbf{v}_{\tau}) \leq 2 \beta (2R)^2$, and then $\hat{\sigma}_{\tau} = \frac{ \nabla_{\tau}^{\top} (\mathbf{x}_{\tau} - \mathbf{v}_{\tau})}{2 \beta (2R)^2}$, and we have
    \begin{align*}
        h_{\tau+1} & \leq h_{\tau} + \hat{\sigma}_{\tau}^2 \beta (2R)^2 - \hat{\sigma}_{\tau}  \nabla_{\tau}^{\top} (\mathbf{x}_{\tau} - \mathbf{v}_{\tau})  \leq h_{\tau} -  \frac{\left( \nabla_{\tau}^{\top} (\mathbf{x}_{\tau} - \mathbf{v}_{\tau})\right)^2}{ 4 \beta (2R)^2} .
    \end{align*}
    From both cases, we have
    \begin{align*}
        h_{\tau+1}  \leq h_{\tau} - \min \bigg{\{} \frac{\left( \nabla_{\tau}^{\top} (\mathbf{x}_{\tau} - \mathbf{v}_{\tau})\right)^2}{ 4 \beta (2R)^2}, \frac{ \nabla_{\tau}^{\top} (\mathbf{x}_{\tau} - \mathbf{v}_{\tau})}{2} \bigg{\}}  & \leq h_{1} - \tau \min_{i = 1, \dots, \tau} \bigg{\{}  \frac{ \left(\nabla_{i}^{\top} (\mathbf{x}_{i} - \mathbf{v}_{i})\right)^2}{ 4 \beta (2R)^2} , \frac{ \nabla_{i}^{\top} (\mathbf{x}_{i} - \mathbf{v}_{i})}{2} \bigg{\}} \nonumber \\ 
        & \leq  h_{1} - \tau \min \bigg{\{}  \frac{ \epsilon^2}{4 \beta (2R)^2} , \frac{ \epsilon}{2} \bigg{\}}.
    \end{align*}
    Thus, for all cases, after a maximum of $L$ iterations, when
    \begin{align*}
        L = \max \bigg{\{} \frac{4 \beta (2R)^2}{\epsilon^2} (h_1 - \epsilon) , 
        ~ \frac{2}{\epsilon} (h_1 - \epsilon) \bigg{\}},
    \end{align*}
    we obtain $ h_{L+1} \leq \epsilon$.
\end{proof}

\end{document}

%% file: defs.tex
\newtheorem{theorem} {Theorem}
\newtheorem{lemma} {Lemma}

\def\x{{\mathbf{x}}}

\def\u{{\mathbf{u}}}
\def\v{{\mathbf{v}}}

\def\y{{\mathbf{y}}}

\newcommand{\mK}{\mathcal{K}}
\newcommand{\mS}{\mathcal{S}}
\newcommand{\ball}{\mathcal{B}}

\newcommand{\reals}{\mathbb{R}}